\documentclass[akbc,twoside,11pt,lettersize]{article}
\usepackage{akbc}
\akbcheading
\finalcopy 

\usepackage{subcaption}
\usepackage{wrapfig}

\usepackage{amsthm}
\usepackage{algorithm}
\usepackage{algorithmic}
\usepackage{array}
\usepackage{setspace}
\usepackage{float}
\usepackage{url}
\usepackage{balance}
\usepackage{setspace}
\usepackage{amsmath}
\usepackage{enumitem}
\usepackage{makecell}
\usepackage{multirow}
\usepackage{color,xcolor}
\usepackage{graphicx}
\usepackage{breakcites}
\usepackage{xspace}
\usepackage{xcolor}
\usepackage{amssymb}
\usepackage{pifont}
\usepackage{tabulary}
\usepackage{arydshln}
\usepackage{bm}
    
\usepackage{tabularx, booktabs}
\usepackage{diagbox}
\newcolumntype{Y}{>{\centering\arraybackslash}X}
\newcolumntype{L}[1]{>{\raggedright\let\newline\\\arraybackslash\hspace{0pt}}m{#1}}
\newcolumntype{C}[1]{>{\centering\let\newline\\\arraybackslash\hspace{0pt}}m{#1}}
\newcolumntype{R}[1]{>{\raggedleft\let\newline\\\arraybackslash\hspace{0pt}}m{#1}}

\begin{document}
\title{Partially-Typed NER Datasets Integration: \\Connecting Practice to Theory}

\ShortHeadings{Partially-Typed NER Datasets Integration: \\Connecting Practice to Theory}{}

\author{\small
Shi Zhi\textsuperscript{$\ddagger$}\thanks{Equal Contributions.}\;\,\,
Liyuan Liu\textsuperscript{$\ddagger$}$^{*}$\,
Yu Zhang\textsuperscript{$\ddagger$}\,
Shiyin Wang\textsuperscript{$\ddagger$}\,
Qi Li\textsuperscript{$\mathsection$}\,
Chao Zhang\textsuperscript{$\dagger$}\,
Jiawei Han\textsuperscript{$\ddagger$}\\
\textsuperscript{$\ddagger$}{\footnotesize University of Illinois at Urbana-Champaign}, \texttt{\footnotesize \{shizhi2 ll2 yuz9 shiyinw hanj\}@illinois.edu}\\
\textsuperscript{$\mathsection$}{\footnotesize Iowa State University}, \texttt{\footnotesize{qli@iastate.edu}}\\
\textsuperscript{$\dagger$}{\footnotesize Georgia Tech}, \texttt{\footnotesize{chaozhang@gatech.edu}}}

\newcommand{\ours}{EvolvT}
\newcommand{\vsafc}{\vspace{-0.5cm}}
\newcommand{\vsaft}{\vspace{-0.5cm}}
\newcommand{\bs}{\boldsymbol}
\newcommand{\mb}{\mathbf}
\newcommand{\comments}[1]{\textcolor{red}{#1}}

\newcommand{\smallsection}[1]{{\noindent \textbf{#1.}}}

\def \eg {\textit{e.g.}}
\def \ie {\textit{i.e.}}

\def \D {\mathcal{D}}
\def \T {\mathcal{T}}

\def \b {\mathbf{b}}
\def \c {\mathbf{c}}
\def \f {\mathbf{f}}
\def \g {\mathbf{g}}
\def \h {\mathbf{h}}
\def \o {\mathbf{o}}
\def \p {\mathbf{p}}
\def \r {\mathbf{r}}
\def \v {\mathbf{v}}
\def \w {\mathbf{w}}
\def \x {\mathbf{x}}
\def \y {\mathbf{y}}

\def \F {\mathbf{F}}
\def \G {\mathbf{G}}
\def \X {\mathbf{X}}
\def \Y {\mathbf{Y}}

\newtheorem{exmp}{Example}
\newtheorem{theo}{Theorem}
\newtheorem{lemm}{Lemma}

\maketitle

       
\begin{abstract}
While typical named entity recognition (NER) models require the training set to be annotated with all target types, each available datasets may only cover a part of them. 
Instead of relying on fully-typed NER datasets, many efforts have been made to leverage multiple partially-typed ones for training and allow the resulting model to cover a full type set. 
However, there is neither guarantee on the quality of integrated datasets, nor guidance on the design of training algorithms. 
Here, we conduct a systematic analysis and comparison between partially-typed NER datasets and fully-typed ones, in both theoretical and empirical manner.  
Firstly, we derive a bound to establish that models trained with partially-typed annotations can reach a similar performance with the ones trained with fully-typed annotations, which also provides guidance on the algorithm design. 
Moreover, we conduct controlled experiments, which shows partially-typed datasets leads to similar performance with the model trained with the same amount of fully-typed annotations.
\end{abstract}

\section{Intorduction}

\begin{wrapfigure}{r}{0.5\textwidth} 
\vspace{-0.5cm}
\centering
\includegraphics[width=0.5\textwidth]{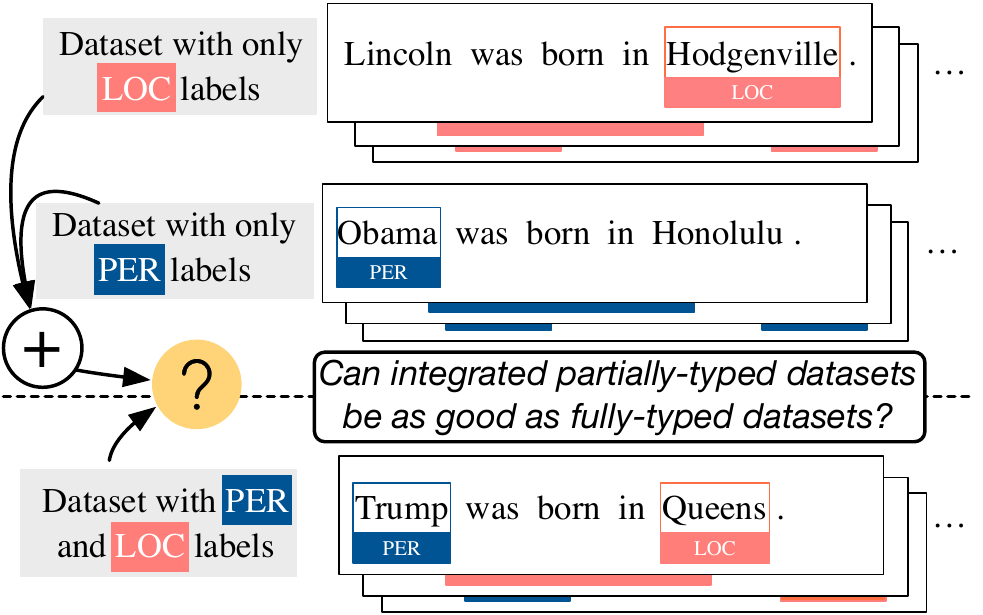}
\caption{Our goal is to understand whether the potential of the integrated partially-typed NER datasets is as good as fully-typed one.}\label{fig:problem_setting}
\vspace{-0.5cm}
\end{wrapfigure}

Named entity recognition (NER) aims to identify entities in natural language and serves as the first step in various applications such as information extraction~\cite{bunescu2005shortest} and information retrieval~\cite{guo2009named}.
While traditional methods require one corpus to be annotated with all target types, 
fulfill such requirements is costly and time-consuming, especially for domains requiring experts to provide annotations.
At the same time, there may exist multiple datasets where each covers a different subset of target types.
For example, most biomedical NER datasets only contain one or two types per dataset, while real-world applications require NER models covering all entity types. 
Intuitively, these datasets have complementary information to each other, and learning across them allows us to cover all target types.
Despite the great potential of this learning paradigm, it brings us new challenges, \ie, how to design an algorithm to learn across datasets and how good the resulting model can be. 

Annotations from different partially-typed datasets are not only incomplete but heterogeneous (different datasets are annotated to different label space). 
Therefore, special handling is required to integrate these datasets. 
For example, in Figure~\ref{fig:problem_setting}, only LOC entities are annotated in the first dataset and only PER entities are labeled in the second dataset.
Although combining them together provides us annotations for both types of entities, directly merging them together without handling their difference would result in massive label noise (\eg, it teaches the model to annotate \textit{Lincoln} and other person names as none-entity in the first dataset).
Many efforts have been made to mediate different datasets and handle their disparity~\cite{scheffer2001active,mukherjee2004taming,bui2008learning,yu2009learning,fernandes2011learning,greenberg2018marginal,jie2019better,beryozkin2019joint,huang2019learning}.
For example, \citet{greenberg2018marginal} proposes to combine EM marginalization and conditional random field (CRF), which significantly outperforms its directly merging baseline and verifies the need of integration algorithm. 
However, due to the lack of theory describing the underlying mechanisms and assumptions of the desired integration, it is still not clear on how dataset integration should be conducted, 
or whether integrated partially-typed datasets can be as good as fully-typed datasets.

Here, we conduct systematic analysis and comparison between partially-typed NER datasets and fully-typed NER datasets. 
Our main contributions are summarized as
follows.
\begin{itemize}[leftmargin=*]

    \item We conduct theoretical analysis on the potential of partially-typed datasets. 
    We establish that, under certain conditions, model performance of detecting fully-typed entities is bounded by the performance of detecting partially-typed entities.
    Thus, with partially-typed annotations, the algorithm is able to optimize both a upper bound and a lower bound of the fully-typed performance. 
    Our results not only reveal partially-typed datasets have a similar potential with fully-typed ones but shed insights on the algorithm design. 
    \item We further designed controlled experiments to verify our theoretical analysis. 
    Specifically, we generated simulated partially-typed and fully-typed datasets, implemented various heuristic methods for dataset integration. 
    With the same amount of annotated entities, models trained with partially-typed datasets achieve the comparable performance with ones trained with fully-typed datasets. 
    It matches our theoretical results and verifies the potential of partially-type datasets. 
\end{itemize}

\section{Partilly-Typed Dataset Analysis}
\label{sec:theory}

Here, we conduct analysis on the potential of the partially-typed annotation.
Specifically, we aim to understand whether partially-typed datasets can achieve a similar performance with fully-typed datasets.  
We first describe our problem setting, briefly introduce sequence labeling models, and then proceed to our theoretical results. 

\subsection{Problem Setting}

We consider a collection of annotated NER corpora
$\{\D_1, \cdots, \D_N\}$, where $\D_i$ is a corpus with a pre-defined target type set, 
denoted by $\T_i$.
In other words, for $\D_i$, all entities within $\T_i$ are labeled with their types and other words are labelled with \texttt{O}, \ie, not-target-type.

We aim to leverage all datasets $\{\D_1, \cdots, \D_N\}$ to train a unified named entity recognizer which can extract entities of any type in 
$\T_{all} = \T_1 \cup \T_2 \cup \cdots \cup \T_N$.
Formally, we define our task as: 
constructing a single model with $\{\D_1, \cdots, \D_N\}$ to annotate all entities within the complete type set $\T_{all}$ 
and all other words as \texttt{O}.

\subsection{Sequence Labeling}
\label{subsec:seq_learn}


Recently, most NER systems are characterized by sequence labeling models under labeling schemas like \texttt{IOB} and \texttt{IOBES}~\cite{ratinov2009design}.
In the \texttt{IOBES} schema, when a sequence of tokens is identified as a named entity, its starting, middle and ending tokens are labeled as \texttt{B-}, \texttt{I-}, \texttt{E-}, respectively, followed by the type.
If a single word is an entity, it is labeled as \texttt{S-} instead. 

To capture the dependency among labels, first-order CRF is widely used in existing NER systems and has crucial impact on the model performance~\cite{reimers2017optimal}. 
Specifically, for input sequence $X=\{x_1, x_2, \ldots, x_T\}$, most models calculate a representation for each token. 
We refer the representation sequence as  $\X = \{\x_1, \cdots, \x_T\}$, where $\x_i$ is the representation of $x_i$. 
For an input sequence $\X$, CRF defines the conditional probability of $\y = \{y_1, \cdots, y_T\}$ as 
\begin{equation}
p(\y|X) = \frac{\prod_{t=1}^{T} \phi(y_{t-1}, y_t, \x_t)}{\sum_{\hat{\y} \in \Y(X)} \prod_{t=1}^{T} \phi(\hat{y}_{t-1}, \hat{y}_t, \x_t)},
\label{eqn:crf}
\end{equation}
where $\hat{\y} = \{\hat{y}_1, \cdots, \hat{y}_T\}$ is a generic label sequence, $\Y(X)$ is the set of all generic label sequences for $\X$ and $\phi(y_{t-1}, y_t, \x_t)$ is the potential function. For computation efficiency, the potential function is usually defined as:
$$
\phi(y_{t-1}, y_t, \x_t) = \exp(W_{y_{t-1}, y_t} \x_t + b_{y_{t-1}, y_t}),
$$
where $W_{y_{t-1}, y_t}$ and $b_{y_{t-1}, y_t}$ are the weight and bias term to be learned.

The negative log-likelihood of Equation~\ref{eqn:crf} is treated as the loss function for model training; during inference, the optimal label sequence is found by maximizing the probability in Equation~\ref{eqn:crf}.
Although $|\Y(X)|$ is exponential w.r.t. the sentence length, both training and inference can be efficiently completed with dynamic programming algorithms.

\subsection{Potential of Partially-typed Datasets}

Intuitively, from multiple partially annotated corpora, we could obtain sufficient supervision for training w.r.t. the complete type set $\T_{all}$. 
Still, there is no guarantee about the model performance trained with such annotations.

Existing studies about generalization errors show that, model performance w.r.t. $\T_i$ is bounded by its performance on the corresponding training set $\D_i$~\cite{london2016stability}.
Here, we further establish the connection between model performance w.r.t. $\T_i$ and the performance w.r.t. $\T_{all}$.
In particular, since typical sequence labeling models treat sentences as the training unit, we use the sentence-level error rate~\cite{fernandes2011learning} to evaluate the performance w.r.t. $\T_{all}$.
\begin{align}
E_{all} = \mathbb{E}_{X, \y} [\mathbb{I}(\varphi(X, \y) \leq \max_{\hat{\y} \in \Y(X) - \y } \varphi(X, \hat{\y}))],
\label{eqn:e_all}
\end{align}
where $\varphi(X, \y)$ is the score calculated by the NER model for the input $X$ and label $\y$. 
Specifically, for the CRF model (Equation~\ref{eqn:crf}), $\varphi(X, \y)$ would be $\prod_{t=1}^{T} \phi(y_{t-1}, y_t, \x_t)$. 
$\mathbb{I}(.)$ is an error indicator function,
it equals to zero if its condition is false, \ie, the correct label sequence has the highest score; or equals to one otherwise.  

Next, we define the partial error rate, which only considers entity prediction within $\T_i$ and can be evaluated with the partially-typed dataset $\D_i$. 
Since only entities within $\T_i$ are evaluated,
predicted entities with other types need to be omitted. 
Thus, $\T_i(\y)$ is defined to replace all annotations in $\T_{all} - \T_i$ with \texttt{O} and convert $\y$ into a new label sequence.
Then, with regard to a type set $\T_i$, the target label sequence set can be defined as 
$$
\Y_{\T_i}(X, \y) = \{\hat{\y} | \hat{\y} \in \Y(X), \T_i(\hat{\y}) = \T_i(\y)\}
$$
It is worth mentioning that $\Y_{\T_i}(X, \y)$ would result in the same set of sequences no matter whether $\y$ is fully-typed or partially-typed, thus can be calculated with either fully-typed dataset or $\D_i$.
Based on $\Y_{\T_i}(X, \y)$, we define the sentence-level partial error rate w.r.t. $\T_i$ as:
\begin{align}
E_{\T_i} &= \mathbb{E}_{X, \y} [\mathbb{I}( \max_{\hat{\y} \in \Y_{\T_i}(X, \y)} \varphi(X, \hat{\y}) \leq \max_{\hat{\y} \in \Y(X) - \Y_{\T_i}(X, \y)} \varphi(X, \hat{\y}))]
\label{eqn:eti}
\end{align}
Specifically, the error indicator function $\mathbb{I}(.)$ in Equation~\ref{eqn:eti} would equal to zero if there exists a label sequence in $\Y_{\T_i}(X, \y)$ having the highest score among all possible label sequences $Y(X)$; otherwise, it would equal to one. 
Similar to $\Y_{\T_i}(X, \y)$, $E_{\T_i}$ 
would have the same value evaluated on fully-typed datasets and on the partially-typed dataset $\D_i$.
To simplify the notation, we use full annotations to calculate both $E_{\T_i}$ and $E_{all}$, then derive their connections as

\begin{theo}
\label{the:main}
If annotations w.r.t. $\{\T_1, \cdots, \T_N\}$ could determine annotations w.r.t. $\T_{all}$ without any ambiguity, \ie, 
for 
$\forall\, X \mbox{ and } \y, \{\y\} = \Y_{\T_1}(X, \y) \cap \cdots \cap \Y_{\T_N}(X, \y)$. 
We would have:
$$
\frac{\sum_{i = 1}^N E_{\T_i}}{N} \leq E_{all} \leq \sum_{i = 1}^N E_{\T_i}
$$
\end{theo}
Theory~\ref{the:main} shows that, by training model with partially-typed datasets, an algorithm can optimize both the lower bound and the upper bound of the performance for recognizing fully-typed entities. 

To prove this theory, we first introduce a lemma.
Specifically, 
We denote the prediction error as:
$$
\delta(X, \y) = \mathbb{I}(\varphi(X, \y) \leq \max_{\hat{\y} \in \Y(X) - \y } \varphi(X, \hat{\y})).
$$
In this way, we have 
$
E_{all} =  \mathbb{E}_{X, \y} [\delta(X, \y)].
$
Similarly, we define $\delta_{\T_i}(X, \y)$, the prediction error within $\T_i$ as
\begin{align*}
\mathbb{I}( \max_{\hat{\y} \in \Y_{\T_i}(X, \y)} \varphi(X, \hat{\y})
\leq \max_{\hat{\y} \in \Y(X) - \Y_{\T_i}(X, \y)} \varphi(X, \hat{\y})),
\end{align*}
and get $E_{\T_i} = \mathbb{E}_{X, \y} [\delta_{\T_i}(X, \y)]$. 
Now we can establish the connection between $\delta(X, \y)$ and $\delta_{\T_i}(X, \y)$ as follows.

\begin{lemm}
\label{lam:max}
If annotations of $\{\T_1, \cdots, \T_N\}$ could determine annotations of $\T_{all}$ without any ambiguity, \ie, $\forall\, \X \mbox{ and } \y, |\Y_{\T_1}(X, \y) \cap \cdots \cap \Y_{\T_N}(X, \y)| = 1$, we could have: $\forall \, X$ and $\y$,  the value of $\delta(X, \y)$ would be equal to $\max( \delta_{\T_1}(X, \y)), \cdots, \delta_{\T_N}(X, \y) )$.
\end{lemm}

\begin{proof}
If $\delta(X, \y) = 0$, we have $$\varphi(X, \y) > \max_{\hat{\y} \in \Y(X) - \y } \varphi(X, \hat{\y}).$$
Thus for $\forall \mbox{ and } \T_i , \y \in \Y_{\T_i}$, we have $\delta_{\T_i} (X, \y) = 0$. 
Then, we have
\begin{equation}
\max( \delta_{\T_1}(X, \y)), \cdots, \delta_{\T_N}(X, \y) ) = 0.
\label{eqn:max_e_0}
\end{equation} 
Now we proceed to prove that, if Equation~\ref{eqn:max_e_0} holds true, we have $\delta(X, \y) = 0$. For the sake of contradiction, we suppose $\exists X, \y$ that satisfies both $\delta(X, \y) = 1$ and Equation~\ref{eqn:max_e_0}. 
Thus, $\forall \, \T_i, \, \delta_{\T_i}(X, \y) = 0$.
Since $\delta(X, \y) = 1$, we  have 
$$
\forall \, \hat{\y} \in \Y(X), \varphi(X, \y^*) \geq \varphi(X, \hat{\y}),
$$ 
where $\y^* = \arg\max_{\hat{\y} \in \Y(X) - \y } \varphi(X, \hat{\y})$. 
Besides, since $
|\Y_{\T_1}(X, \y) \cap \cdots \cap \Y_{\T_N}(X, \y)| = 1,
$
we have $\exists \, \T_{\y^*} \mbox{ that }\, \y^* \not\in \Y_{\T_{\y^*}}(X, \y)$. 
Therefore, we have $\y^* \in \Y(X, \y) - \T_{\y^*}(X, \y)$.
Thus,
$$
\max_{\hat{\y} \in \Y_{\T_i}(X, \y)} \varphi(X, \hat{\y}) \leq  \varphi(X, \y^*) = \max_{\hat{\y} \in \Y(X) - \Y_{\T_i}(X, \y)} \varphi(X, \hat{\y}).
$$
Based on definition, we have $\delta_{\T_i}(X, Y) = 1$, which contradicts to our assumption that Equation 3 holds true.

Hence, if $\max( \delta_{\T_1}(X, \y)), \cdots, \delta_{\T_N}(X, \y) ) = 0$, $\delta(X, \y) = 0$.
Accordingly, $\forall \, X$ and $\y$, we have 
$$
\delta(X, \y) = \max( \delta_{\T_1}(X, \y)), \cdots, \delta_{\T_N}(X, \y) ).
$$
\end{proof}
\noindent
Now, we proceed to prove the Theorem~\ref{the:main}.
\begin{proof}
Since for $\forall\, \T_i \mbox{ we have } \y \in \Y_{\T_i}(X, \y)$, thus
\begin{align*}
\varphi(X, \y)  -  \max_{\hat{\y} \in \Y(X) - \y}  \varphi(X, \hat{\y}) \leq
\max_{\hat{\y} \in \Y_{\T_i}(X, \y)} \varphi(X, \hat{\y}) - \max_{\hat{\y} \in \Y(X) - \Y_{\T_i}(X, \y)} \varphi(X, \hat{\y})
\end{align*}
Based on the definition, we get $E_{\T_i} \leq E_{all}$ and 
$\frac{\sum_{i = 1}^N E_{\T_i}}{N} \leq E_{all}$.

Then, we prove the second inequality in Theorem~\ref{the:main}.
With Lemma~\ref{lam:max}, we know that, if annotations of $\{\T_1, \cdots, \T_N\}$ could determine annotations of $\T_{all}$ without any ambiguity, we would have $\delta(X, \y) = \max( \delta_{\T_1}(X, \y)), \cdots, \delta_{\T_N}(X, \y) )$. 
Therefore, we have
\begin{align*}
E_{all} &=  \mathbb{E}_{X, \y} [\delta(X, \y)] = \mathbb{E}_{X, \y} [\max( \delta_{\T_1}(X, \y)), \cdots, \delta_{\T_N}(X, \y) )] \\
&\leq \mathbb{E}_{X, \y} [\sum_{i=1}^N \delta_{\T_i}(X, \y))] =\sum_{i=1}^N \mathbb{E}_{X, \y} [\delta_{\T_i}(X, \y))] = \sum_{i=1}^N E_{\T_i}
\end{align*}
which is the second inequality in Theorem~\ref{the:main}.

\end{proof}

For a quick summary, we establish that,
if partially-annotated datasets $\{\T_i\}$ provide sufficient supervision (\ie, annotations for $\{\T_1, \cdots, \T_N\}$ could determine annotations for $\T_{all}$ without any ambiguity), $E_{all}$ on the test set can be bounded by $\{E_{\T_i}\}$ on the test set; and based on previous work~\cite{london2016stability}, we can bound $E_{\T_i}$ on the test set with $E_{\T_i}$ on $\D_i$. 
In other words, for a list of datasets, if their labeling schemas $\{ \T_1, \cdots, \T_N\}$ could determine the annotation for $\T_{all}$ without any ambiguity, 
partial annotations alone can be sufficient for sequence labeling model training. 

It is worth mentioning that our study is the first analyzing the potential of partially-typed datasets for sequence labeling and providing theoretical supports for learning with complementary but partially-typed datasets.
We now discuss its connection to the previous work, which studies this problem empirically.

\subsection{Connection to Existing Methods}

Due to the lack of guidance on theory design, many principles and methods have been leveraged to integrate partially-typed datasets (as introduced in Section~\ref{subsec:partial-related}). 
\citet{jie2019better} shows that EM-CRF or CRF with marginal likelihood training significantly outperforms other principles. 
Here, we show the connection between our theory and marginal likelihood training. 

\smallsection{Connection to Marginal Likelihood Training}
As introduced in Section~\ref{subsec:seq_learn}, the widely used objective function for conventional CRF models is the negative log likelihood:
\begin{align}
     - \log p(\y | X) = \log (\sum_{\hat{\y} \in \Y(X)} \varphi(X, \hat{\y})) - \log \varphi(X, \y) 
    \approx  \max_{\hat{\y} \in \Y(X)} \log\varphi(X, \hat{\y}) - \log \varphi(X, \y) \label{eqn:app_crf}
\end{align}

Equation~\ref{eqn:app_crf} can be viewed as an approximation to the sentence error rate (\ie, Equation~\ref{eqn:e_all}).
Specifically, if $\varphi(X, \y)$ has the largest value among all possible label sequences (the model output is correct), both $\delta(X, \y)$ and Equation~\ref{eqn:app_crf} would be zero; otherwise (the model output is wrong), $\delta(X, \y)$ would be one and Equation~\ref{eqn:app_crf} would be positive. 
Thus, the sentence-level error rate could be viewed as the $l_0$ norm of Equation~\ref{eqn:app_crf}, and the negative log likelihood can be viewed as an approximate of the sentence level error rate. 

Since $E_{all}$ could be bounded by $\sum_{\T_i} E_{\T_i}$ (Theorem~\ref{the:main}), it's intuitive to design the learning objective as minimizing an approximation of $\sum_{\T_i} E_{\T_i}$ for the complementary learning. 
Specifically, similar to Equation~\ref{eqn:app_crf}, $\mathbb{I}(.)$ in $E_{\T_i}$ is approximated as below
\begin{align*}
     &\max_{\hat{\y} \in \Y(X)} \log\varphi(X, \hat{\y}) - \max_{\hat{\y} \in \Y_{\T_i}(X, \y)} \log\varphi(X, \hat{\y}) \approx \log (\sum_{\hat{\y} \in \Y(X)} \varphi(X, \hat{\y})) - \log (\sum_{\hat{\y} \in \Y_{\T_i}(X, \y)} \varphi(X, \hat{\y})) 
\end{align*}
which is the objective for the marginal likelihood training~\cite{greenberg2018marginal,jie2019better} (as visualized in Fig~\ref{fig:fuzzy}).
Similar to Equation~\ref{eqn:crf}, it can be calculated efficiently with dynamic programming. 



\section{Experiments}
\label{sec:exp}

We further conduct experiments to compare fully-typed datasets with partially-typed ones.
We first introduce compared methods, then discuss the empirical results.

\subsection{Experiment Setting}

We generate simulated partially-typed datases from fully-typed datasets, which allows us to have consistent annotation quality and controllable annotation quantity.
Specifically, we first devide a bioNER dataset JNLPBA~\cite{crichton2017neural} into $n$ folds ($n=\{3,4,5\}$ in our experiments), preserve one type of entities in each split and mask other entities as \texttt{O}.
Statistics of the resulting dataset are summarized in Table~\ref{stats} and more details are included in the Appendix.

\begin{table}[t]
\begin{center}
\caption{Number of Sentences and Entities in the Simulated Dataset.}
\label{stats}
\begin{tabularx}{\columnwidth}{l *{10}{Y}}
\hline
 & \multicolumn{2}{c}{DNA} & \multicolumn{2}{c}{protein} & \multicolumn{2}{c}{cell-type} & \multicolumn{2}{c}{cell-line} & \multicolumn{2}{c}{RNA} \\\cmidrule{2-11}
 & Sent. & Entity & Sent. & Entity & Sent. & Entity & Sent. & Entity & Sent. & Entity \\
 \midrule
 3-split 
& 5571 & 7385
& 5571 & 16729
& 5571 & 4778 
& -- & -- & -- & --\\
\midrule
4-split 
&4178 & 5607
&4178 & 12720
&4179 & 3575
&4178 & 2495
& -- & --\\
\midrule
5-split
&3343 & 4509
&3342 & 10189
&3343 & 2855
&3342 & 1940
&3343 & 539\\
\bottomrule
\end{tabularx}
\end{center}
\vsaft
\end{table}


\subsection{Compared Methods}

We implement and conduct experiments with four types of methods.
It is worth mentioning that, all methods are based on the same base sequence labeling model as described in Sec~\ref{subsec:seq_learn}.



\smallsection{Naive Method (Concat)} 
We implement a naive method to demonstrate the difference between our problem to classical supervised learning. 
\textit{Concat} naively overlooks the label inconsistency problem and trains a model by directly combining different datasets. As each separate dataset can be viewed as an incompletely labeled corpus w.r.t. the union type space, \textit{Concat} introduces many false negatives into the training process, undermining the performance of the model.

\smallsection{Marginal Likelihood Training (Partial)}
As in Figure~\ref{fig:fuzzy}, there may exist multiple valid target sequences in a partially-typed dataset. 
Therefore, the marginal likelihood training is employed as the learning objective, which seeks to maximize the probability of all possible label sequences~\cite{greenberg2018marginal}:
$
\mathcal{L} = - \sum_{X, \y} \log (\sum_{\hat{\y} \in \Y_{\T_i}(X, \y)} p(\hat{\y} | X))
$. 
It can be efficiently calculated with dynamic programming.

By maximizing the probability of all label sequences that are consistent with partially-typed labels, the CRF model is allowed to coordinate the label spaces during model learning.
For example, in Figure~\ref{fig:fuzzy}, the marginal likelihood training would try to discriminate \texttt{PROTEIN} against other entities in the BC2GM dataset and discriminate \texttt{CHEMICAL} against other entities in the BC2CHEMD dataset.
This allows the model to take advantage of the positive labels from each dataset 
while over-penalize uncertain labels. 
As a result, the model spontaneously integrates supervision signals from all training sets, having the ability to identify entities of all previously seen types.
In the inference process, such model would extract entities of all types.

\begin{figure*}[t]
\includegraphics[width=\linewidth]{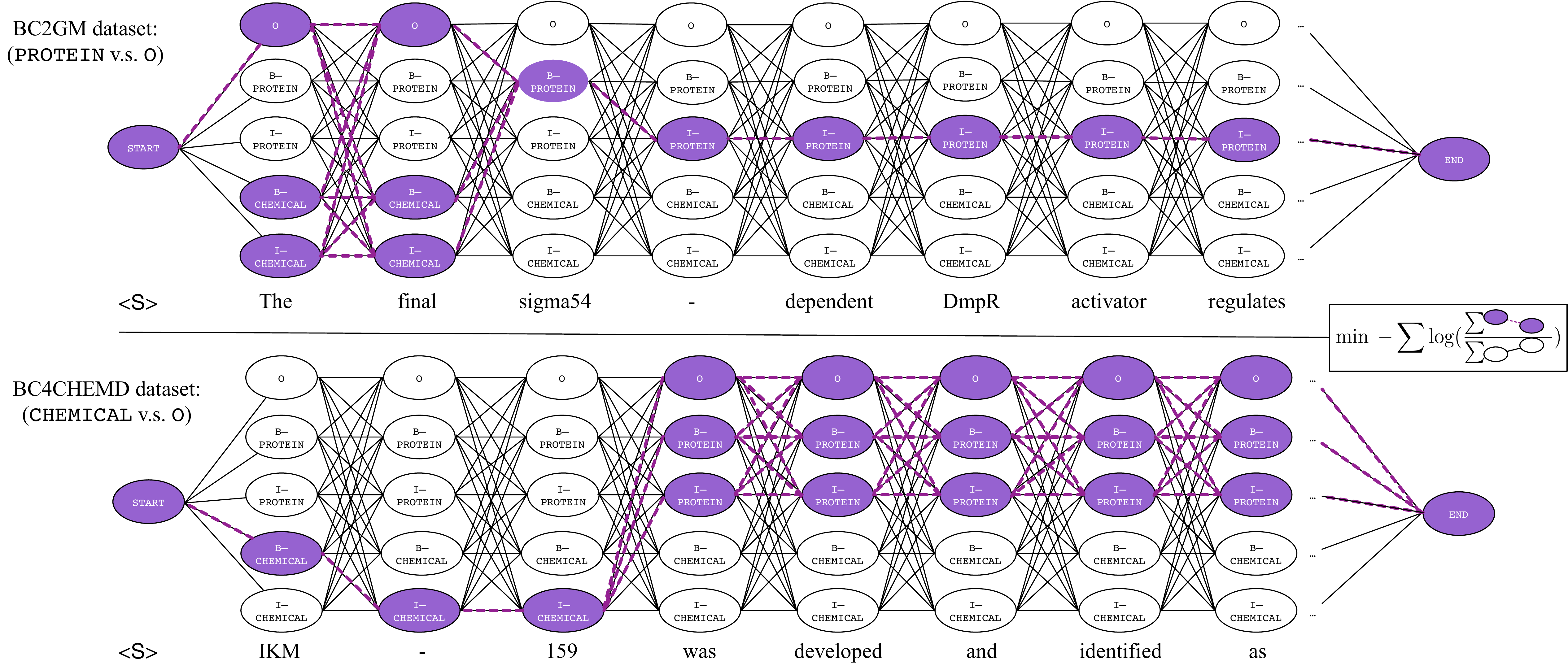}
\caption{Marginal Likelihood Training for CRFs. The BC2GM dataset is annotated with only \texttt{PROTEIN} entities; the BC4CHEMD dataset is annotated with only \texttt{CHEMICAL} entities.}
\label{fig:fuzzy}
\end{figure*}



\smallsection{Label Propagation (Propagate)}
We design a heuristic strategy to integrate partially-typed datasets by propagating labels. 
Specifically, we first train $N$ separate models, each from a different dataset.
Then, we use these models to annotate each other.
Thus, each dataset will thus have two types of labels: gold labels come from manual annotation and propagated labels come from models trained with other datasets.
\textit{Propagate} merges original annotations with propagated labels by treating both as the target sequences and the marginal likelihood as the objective for training. 

\smallsection{Conventional Supervised Method (Standard$^{\spadesuit}$)}
We also construct conventional supervised sequence labeling model with sampled datasets. 
Specifically, \textit{Standard$^{\spadesuit}$($S$ sent. w. $N$ types)} is trained on a fully-typed dataset which has roughly the same amount of annotated entities with $N$-split but less sentences.
\textit{Standard$^{\spadesuit}$(all sent. w. $N$ types)} is to train on a fully-typed dataset which has roughly the same amount of sentences with $N$-split but more annotated entities. 
\textit{1-Type} is to train a single-type NER with the access to all annotations of a single type of entities. 

\subsection{Performance Comparison}

\begin{table}[t]
\caption{Performance Comparison. \textit{Standard$^{\spadesuit}$ ($S$ sent. w. $N$ types)} is a conventional sequence labeling model trained with $s$ sentences (fully annotated with $n$ types of entities).}
\begin{center}
\scalebox{0.95}{
\begin{tabularx}{\columnwidth}{c c *{6}{Y}}
\toprule        
             \multirow{3}{*}{Dataset} & \multirow{3}{*}{Method}  &   \multicolumn{5}{c}{Per-type $F_1$} & \multirow{3}{*}{\begin{tabular}[c]{@{}c@{}}Micro \\ $F_1$\end{tabular}}      \\ \cmidrule{3-7}
                   &                            & cell-type & DNA   & protein & cell-line & RNA   &  \\ \midrule\midrule
\multirow{3}{*}{\begin{tabular}[c]{@{}c@{}} 3-split \\ ($\sim$5571 sent./type)  \end{tabular}} 
& \textit{{Concat}}      & 28.24     & 18.72 & 32.95   & --        & --    & 30.16   \\
                   & \textit{{Propagate}}  & 71.10     & 71.32 &  75.53  & --        & --    & 73.85        \\
                   & \textit{{Partial}}    & \textbf{75.41}     & \textbf{72.55} & \textbf{76.41}   & --        & --    & \textbf{75.58}   \\
                   \cmidrule{1-8}
\multicolumn{2}{c}{\textit{Standard$^{\spadesuit}$ (5571 sent. w. 3 types)}}  & 72.86     &  68.09 &  73.20  & --        & --    & 69.62   \\
\multicolumn{2}{c}{\textit{Standard$^{\spadesuit}$ (all sent. w. 3 types)}}    & 75.64     &  70.69 &  76.37  & --        & --    & 75.47   \\ \midrule\midrule
\multirow{3}{*}{\begin{tabular}[c]{@{}c@{}} 4-split \\ ($\sim$4178 sent./type)\end{tabular}}
& \textit{{Concat}}   & 17.32     & 25.02 & 24.78   & 16.41     & --    & 22.74   \\
                   & \textit{{Propagate}}  & 69.81     & 67.52 & 73.78   & 55.54     & --    & 70.94        \\
                   & \textit{{Partial}}    & \textbf{76.13}     & \textbf{70.55} & \textbf{75.23}   & \textbf{61.82}     & --    & \textbf{73.83}   \\\cmidrule{1-8}
\multicolumn{2}{c}{\textit{Standard$^{\spadesuit}$ (4178 sent. w. 4 types)}}   & 72.30     &  66.09 &  72.19  & 57.78         & --    & 70.09   \\
\multicolumn{2}{c}{\textit{Standard$^{\spadesuit}$ (all sent. w. 4 types)}}    & 78.00     & 70.20 & 76.19   & 62.26     & --    & 74.96   \\ \midrule\midrule
\multirow{3}{*}{\begin{tabular}[c]{@{}c@{}} 5-split \\ ($\sim$3343 sent./type)\end{tabular}}& \textit{{Concat}}   & 6.99      & 24.58 & 19.41   & 9.68      & 13.95 & 16.89   \\
                   & \textit{{Propagate}}  & 68.66     & 63.67 & 70.58   & 56.77     & 68.34 & 68.37        \\
                   & \textit{{Partial}}    & \textbf{73.72}     & \textbf{71.34} & \textbf{74.67}   & \textbf{60.87}     & \textbf{72.65} & \textbf{73.00}   \\\cmidrule{1-8}
\multicolumn{2}{c}{\textit{Standard$^{\spadesuit}$ (3343 sent.  w. 4 types)}}   & 74.32     & 66.67 &  71.15  & 59.02       & 63.96   & 70.49   \\
\multicolumn{2}{c}{\textit{Standard$^{\spadesuit}$ (all sent.  w. 5 types)}}   & 76.71     & 70.64 & 76.31   & 61.08     & 73.07 & 74.75   \\ 
\bottomrule
\end{tabularx}}
\end{center}
\label{tab:perf345}
\vsaft
\end{table}
As shown in Table \ref{tab:perf345}, the \textit{Partial} model significantly outperforms \textit{Concat} and \textit{Propagate} in both per-type and micro $F_1$.
Also, the performance of partially-typed datasets keeps dropping as the number of types grows (since larger $N$ results in less annotated entities per type). 
Among these methods, \textit{Concat} performs much worse than all other methods due to the effect of false negatives, which verifies the importance of properly integrating datasets.
\textit{Propagate} adds more constraints to the target label set of \textit{Partial} and achieves worse performance (\textit{Partial} achieves $2.34\%$, $4.07\%$, $6.77\%$ gain over \textit{Propagate} for $n=3,4,5$, respectively). 
Such result indicate that it is more profitable not to constraints on the target label space and allow the algorithm to infer the labels in an adaptive manner.

To examine the efficacy of partially-typed datasets, we compare the performance of \textit{Partial} and \textit{Standard$^{\spadesuit}$}. 
\textit{Partial} obtains much better results than \textit{Standard$^{\spadesuit}$ ($S$ sent. w. $N$ types)}, which has less sentences and roughly the same number of annotated entities.
At the same time, \textit{Partial} obtains comparable performance than \textit{Standard$^{\spadesuit}$(all sent, $N$ types)}, which has with the same amount of sentences and more annotated entities (comparing to \textit{Standard$^{\spadesuit}$(all sent, $N$ types)}, \textit{Partial} is only trained with $\sim1/N$ annotated entities).
Overall, the result verifies the potential of partially-typed datasets and accords with our derivation and matches our theoretical results.

\section{Related Work}
\label{sec:related}

There exist two aspects of related work regarding the
topic here, which are sequence labeling and training with incomplete annotations.

\subsection{Sequence Labeling}

Most recent approaches to NER have been characterized by the sequence labeling model, i.e., assigning a label to each word in the sentence.
Traditional methods leverage handcrafted features to capture linguistic signals and employ conditional random fields (CRF) to model label dependencies~\cite{finkel2005incorporating,settles2004biomedical,leaman2008banner}.
Many efforts have been made to leverage neural networks for representation learning and free domain experts from handcrafting features~\cite{huang2015bidirectional,chiu2016named,lample2016neural,ma2016end, liu2017empower}.
Recent advances demonstrated the great potential of natural language models and extensive pre-training~\cite{Peters2018DeepCW,akbik2018coling,Devlin2019BERTPO}.
However, most methods only leverage one dataset for training, and doesn't have the ability to handle the inconsistency among datasets. 

\subsection{Training with Incomplete Annotations}
\label{subsec:partial-related}

As aforementioned, special handlings are needed to integrate different datasets. 
To the best of our knowledge, \citet{scheffer2001active} is the first to study the problem of using partially labelled datasets for learning information extraction models.
\citet{mukherjee2004taming} further incorporates ontologies into HMMs to model incomplete annotations. 
\citet{bui2008learning} extends CRFs to handle this problem, while \citet{yu2009learning} leverages structured SVMs.
\citet{carlson2009learning} conduct training with partial perceptron, which only considers typed entities but not none-entity words; transductive perceptron is leveraged to infer labels for all words~\cite{fernandes2011learning}. 
After neural networks have demonstrated their ability to substantially increase the performance of NER models, attempts have been made to train neural models with partially-typed NER dataset~\cite{Giannakopoulos2017UnsupervisedAT,shang2018learning}.
\citet{greenberg2018marginal} studies the problem of using multiple partially-typed NER dataset to conduct training.
More and more attentions have been attracted to further improve the model performance~\cite{jie2019better,beryozkin2019joint,huang2019learning}. 
Our work aims to answer the question whether models trained with partially-typed NER datasets can achieve comparable performances with ones trained with fully-typed NER datasets and how to integrate partially-typed NER datasets.

\section{Conlcusion}
\label{sec:conclu}

We study the potential of partially-typed datasets for NER training. 
Specifically, we theoretically establish that the model performance evaluated with full annotations could be bounded by those with partial annotations.
Besides, we reveal the connection from our derivation to marginal likelihood training, which provides guidance on algorithm design. 
Furthermore we conduct empirical studies to verify our intuition and find experiments match our theoretical results.

There are many interesting directions for future work. 
For example, for even the same entity type, different real-world datasets may have different definitions or annotation guidelines.
It would be beneficial to allow models self-adapt to different label spaces and handle the disparity among different annotation guidelines. 
	
\bibliographystyle{plainnat}
\bibliography{cited}

\newpage
\appendix

\section*{Model Training and Hyper-parameters}
We choose hyper-parameters based on the previous study \cite{reimers2017optimal} and use the chosen values for all methods.
Specifically, we use 30-dimension character embeddings and 200-dimension pre-trained word embeddings \cite{moen2013distributional}. Both LSTMs are bi-directional.
Character-level LSTMs are set to be one-layered with 64-dimension hidden states in each direction; word-level ones are set to be two-layer with 50-dimension hidden states in each direction.
Dropout with ratio 0.25 is applied to the output of each layer.
Stochastic gradient descent with momentum is employed as the optimization algorithm, and the batch size, momentum and learning rate are set to 32, 0.9 and $\eta_t = \frac{\eta_0}{1 + \rho t}$ respectively.
Here $\eta_0 = 0.015$ is the initial learning rate and $\rho = 0.05$ is the decay ratio.
For better stability, we set the gradient clipping threshold to 5.

\section*{Evaluation Metrics and Data Source}
We use entity-level $F_1$ as the evaluation metric. \textbf{Micro $F_1$} is computed over all types by counting the total true positives, false negatives and false positives. \textbf{Per-type $F_1$} is calculated by counts of each type. 
All datasets we use are collected by \cite{crichton2017neural}, can be downloaded from Github\footnote{https://github.com/cambridgeltl/MTL-Bioinformatics-2016}, and the original train, development and test splits are adopted.

\end{document}